\documentclass{article}

\usepackage[preprint]{neurips_2026}

\usepackage{amsmath,amsfonts,bm}

\def\eqref#1{equation~\ref{#1}}

\def\plaineqref#1{\ref{#1}}

\def\1{\bm{1}}

\DeclareMathAlphabet{\mathsfit}{\encodingdefault}{\sfdefault}{m}{sl}
\SetMathAlphabet{\mathsfit}{bold}{\encodingdefault}{\sfdefault}{bx}{n}

\usepackage[hypertexnames=false]{hyperref}
\usepackage{url}
\hypersetup{
  hidelinks,
  pdftitle={Safe to Stop? Risk-Constrained Stopping for Sequential Clinical Diagnosis Agents},
  pdfauthor={Yuexin Wu, Vasile Rus},
  pdfsubject={GenAI4Health 2026 Research Paper Draft}
}
\usepackage{booktabs}
\usepackage{multirow}
\usepackage{amsmath,amssymb,amsthm}
\usepackage{xcolor}
\usepackage{enumitem}
\usepackage{graphicx}
\usepackage{array}
\usepackage{float}

\newtheorem{theorem}{Theorem}
\newtheorem{proposition}{Proposition}

\newcommand{\cros}{\textsc{Cros}}
\newcommand{\crosdet}{\textsc{Cros}-Det}
\newcommand{\crosmix}{\textsc{Cros}-Mix}
\newcommand{\singlecandidate}{Single-candidate test}
\newcommand{\fixedsequenceltt}{Fixed-sequence LTT}
\newcommand{\holmtesting}{Holm testing}
\newcommand{\bonferronitesting}{Bonferroni testing}
\newcommand{\uniformmixture}{Uniform-weight mixture}
\newcommand{\nohistoryranker}{No-history ranker}
\newcommand{\nativescoreranker}{Native-score ranker}
\newcommand{\risk}{\mathcal{R}}
\newcommand{\coverage}{\mathcal{C}}
\newcommand{\cost}{\mathcal{J}}

\title{Safe to Stop? Risk-Constrained Stopping\\
for Sequential Clinical Diagnosis Agents}

\author{
  Yuexin Wu \\
  Department of Computer Science \\
  University of Memphis \\
  \texttt{ywu10@memphis.edu}
  \And
  Vasile Rus \\
  Department of Computer Science \\
  University of Memphis \\
  \texttt{vrus@memphis.edu}
}

\begin{document}

\maketitle

\begin{abstract}
Clinical diagnosis agents must decide not only what test to request next, but also when to diagnose or defer. Existing agent benchmarks largely evaluate accuracy after fixed or unconstrained interaction, leaving autonomous stopping reliability implicit. We present \cros{}, a risk-constrained stopping layer combining state-wise error ranking, policy design on disjoint development splits, and LTT-style exact tests of selective diagnostic error and minimum autonomous coverage for complete sequential policies. Its finite-sample guarantee requires the candidate family, testing rule, and any randomization to be frozen before calibration labels are accessed. On a 1,834-episode MIMIC-derived abdominal-pain benchmark, the full ranker achieves exploratory state-error AUROC 0.853, compared with 0.715 for maximum class probability and 0.552 for the backbone's native stop score. On the previously viewed 367-episode evaluation split, analytically averaging over the frozen \crosmix{} weights yields 16.9\% selective error at 78.8\% coverage, cost 5.57, and 0.68 tests, versus 30.8\% error at 100\% coverage, cost 8.14, and 1.53 tests under native stopping. Forced continuation is non-monotone: error is 28.3\% with HPI alone and 34.3\% after full workup. However, the uniform-weight mixture ablation is cheaper on this viewed split despite missing the locked development margins, and \crosmix{} nominally satisfies the joint criterion in only 6 of 20 development resplits. Because evaluation labels were inspected during earlier development, these findings provide exploratory feasibility and audit evidence, not a confirmatory safety certificate.
\end{abstract}

\section{Introduction}

Language-model agents can generate differential diagnoses, request tests, read new findings, and revise their hypotheses. Recent systems and benchmarks have made this interaction increasingly realistic \citep{hager2024limitations,schmidgall2024agentclinic,liu2024medchain,baniharouni2026lacdm}. Yet a central decision remains weakly specified: when should the agent stop acquiring information and commit to an autonomous diagnosis? Stopping too early can miss a consequential disease; stopping too late wastes tests and may expose patients to avoidable procedures. Uncalibrated confidence thresholds provide no finite-sample joint risk--coverage guarantee, while unconstrained empirical cost minimization can overfit the selection sample \citep{angelopoulos2025ltt,laufergoldshtein2023pareto}.

We formulate sequential diagnosis as selective, risk-constrained stopping. At each stage the agent observes the history, proposes a diagnosis and next test, and the stopping layer either accepts the diagnosis, continues the shared acquisition trajectory, or defers. The desired policy minimizes resource cost subject to two population constraints: conditional diagnostic error among autonomous decisions is at most $\alpha$, and autonomous coverage is at least $\gamma$. This differs from ordinary selective classification \citep{geifman2017selective,geifman2019selectivenet}: prediction quality and the information state both evolve over an agent trajectory.

Our method, \cros{} (Clinical Risk-constrained Optimal Stopping), separates representation, design, and calibration. Here, ``optimal'' means cost-minimizing within a finite, selection-frozen family of threshold-and-horizon policies; \cros{} does not solve unrestricted Bellman optimal stopping or optimize the backbone's test-acquisition policy. First, a risk ranker scores the likelihood that the backbone's current diagnosis is wrong. Second, a selection split freezes the candidate family and, optionally, a randomized mixture chosen to minimize expected cost. Third, a prospectively fresh calibration split can supply exact binomial $p$-values for error and coverage. Multiple testing procedures from learn-then-test (LTT) \citep{angelopoulos2025ltt} convert candidate-wise tests into a finite-sample guarantee. The stopping layer is backbone-agnostic and auditable: each action, observation, risk score, stop/continue/defer outcome, and calibration statistic is logged. In the present study, previously opened labels mean that the same calculations are exploratory calibration checks, not a realized prospective guarantee.

We make four contributions:
\begin{enumerate}[leftmargin=*,itemsep=1pt,topsep=2pt]
    \item We adapt and operationalize LTT-style joint testing of selective diagnostic error and minimum autonomous coverage for complete sequential stopping policies, obtaining an exact finite-sample guarantee when deterministic or episode-wise randomized policies are frozen before calibration.
    \item We instantiate episode-wise randomized mixtures of deterministic stopping policies; Proposition~\ref{prop:sparse-validity} shows that an optimal basic feasible mixture requires at most three component policies and remains compatible with exact calibration when randomized independently across episodes. 
    \item We construct an auditable retrospective benchmark with 1,834 MIMIC-derived ED episodes, 12 nonuniform-cost actions, recorded-result missingness, and a common-backbone comparison that isolates stopping from diagnosis and test-proposal quality.
    \item We report both favorable and negative evidence: the full ranker out-ranks simple scores but weakens under diagnosis-language masking, forced continuation is non-monotone, optimized mixing does not dominate uniform mixing on evaluation, 20-resplit feasibility is unstable, and aggregate control does not ensure subgroup safety.
\end{enumerate}

\section{Related Work}

\paragraph{Clinical decision agents.}
MIMIC-CDM evaluates LLMs on sequential clinical decision making and exposes limitations in diagnostic reasoning and tool use \citep{hager2024limitations}. AgentClinic provides a multimodal simulated clinical environment \citep{schmidgall2024agentclinic}; MedChain emphasizes interactive, sequential clinical benchmarking \citep{liu2024medchain}; and DxChain uses panoramic profiling and adversarial debate \citep{lv2026dxchain}. LA-CDM trains hypothesis and decision agents with reinforcement learning to choose tests and update diagnoses \citep{baniharouni2026lacdm}. These systems motivate the same acquisition loop as our benchmark. Our goal is complementary: we hold the diagnostic backbone and its action trajectory fixed when possible, then evaluate whether a stopping controller can provide a testable population guarantee.

\paragraph{Resource-aware sequential diagnosis.}
MAI-DxO evaluates diagnostic accuracy jointly with the cost of adaptively requested tests \citep{nori2025sequential}. ACTMED uses Bayesian experimental design to select the next test \citep{estevez2025timely}, cost-sensitive reinforcement learning learns adaptive test-panel policies \citep{yu2023deep}, and latent diagnostic trajectory learning trains planning and diagnostic agents to acquire evidence along learned paths \citep{shen2026latent}. These methods can change which tests are selected and therefore change the trajectory. \cros{} is not an end-to-end acquisition method: its common-path design holds backbone diagnoses, test proposals, and forced-continuation trajectories fixed to isolate whether the controller stops, continues with the backbone-proposed test, or defers.

\paragraph{Selective prediction and risk control.}
Selective classifiers abstain on uncertain examples to trade coverage for conditional error \citep{geifman2017selective,geifman2019selectivenet}. Distribution-free risk-controlling prediction sets and conformal risk control extend calibration beyond marginal coverage \citep{bates2021rcps,angelopoulos2024crc}. LTT turns risk constraints into hypothesis tests and controls the probability of selecting an invalid procedure from a finite family \citep{angelopoulos2025ltt}. Geometry-Calibrated Conformal Abstention gives finite-sample guarantees for both participation and correctness of emitted open-ended language-model responses \citep{xu2026geometry}, while SCoRE uses conformal e-values to control a general bounded risk among selected outputs \citep{bai2026score}. Thus neither selective risk control nor participation guarantees are new in isolation. We adapt these ideas to a stateful clinical stopping problem whose complete, frozen policy is jointly tested for conditional diagnostic error and a lower autonomous-coverage bound, with resource costs used for policy design and comparison rather than included in the validity claim. The exact guarantee concerns the policy chosen before calibration, not the accuracy of the learned risk score.

\paragraph{Clinical stopping and abstention.}
Uncertainty-aware abstention has been studied for static medical-text prediction \citep{vazhentsev2025abstention}, while Safe-Psych evaluates diagnose, clarify, and abstain decisions as psychiatric evidence is revealed incrementally \citep{presacan2026safepsych}. MediQ studies interactive clinical diagnosis in which a model refrains from diagnosing under insufficient information and asks follow-up questions, but does not provide finite-sample joint control of selective diagnostic error and autonomous coverage \citep{li2024advances}. Foo and Chang formulate staged clinical prediction as an expected-loss optimal-stopping problem with explicit decision and testing costs and Bellman recursion \citep{foo2026optimal}. Their objective optimizes whether expected decision value justifies further testing; it does not provide \cros{}'s exact joint test of conditional diagnostic error and minimum autonomous coverage. \cros{} instead freezes a complete sequential policy selected on disjoint development data and then subjects that policy to the joint test. This distinction is methodological, not a claim that Bellman stopping, interactive clinical diagnosis, or clinical abstention is new.

\paragraph{Calibrated sequential stopping and acquisition.}
LTT-style finite-sample risk control is not new to this work. Pareto Testing combines multi-objective design and multiple testing \citep{laufergoldshtein2023pareto}, while accumulated-accuracy-gap control gives distribution-free stopping rules for early time classification \citep{ringel2024early}. MiCP allocates error budgets across turns of retrieval, tool-use, and reasoning workflows, enabling adaptive early stopping with overall conformal coverage while reducing turns and inference cost \citep{zhou2026adaptive}. Its guarantee concerns coverage of the final prediction set in multi-turn reasoning; \cros{} instead tests selective diagnostic error and minimum autonomous coverage for a frozen stop--continue--defer controller on a common clinical trajectory. Other recent work studies selective conformal risk control \citep{xu2025selective}, inference-time reasoning under a compute budget \citep{wang2026conformal}, and post-acquisition recalibration when additional evidence can be requested \citep{xu2026look}. These studies already establish important forms of sequential calibration, joint utility--risk design, or cost-aware acquisition. Our narrower contribution is their integration into sequential clinical diagnosis: a state-wise clinical risk ranker, a lower autonomous-coverage constraint, a common-path MIMIC benchmark, and episode-wise mixtures. The sparsity of the mixture is a standard linear-program consequence rather than a new optimization theorem.

\section{Risk-Constrained Sequential Diagnosis}

\subsection{Problem setup}

An episode $Z=(X_0,Y,O_{1:H})$ contains an initial presentation $X_0$, a reference diagnosis $Y\in\{1,\ldots,K\}$, and potential recorded observations along a maximum horizon $H$. Here, $H$ is the maximum number of test-acquisition stages available in an episode ($H=12$ in our benchmark). A policy-specific horizon $h\leq H$ determines the latest stage at which that policy must stop or defer. At stage $t$, the backbone has history $S_t=(X_0,A_{1:t},O_{1:t})$, produces a diagnosis $\widehat{Y}_t$, and proposes the next test $A_{t+1}\in\mathcal{A}$. A stopping controller $\pi$ maps the observed history and the backbone's current proposal to \textsc{stop}, \textsc{continue}, or \textsc{defer}. Continuing reveals the next recorded result along the backbone-proposed trajectory; \cros{} controls whether that proposal is executed but never substitutes a different test. This common-path interaction is summarized in Figure~\ref{fig:pipeline}. Let $T_\pi$ denote the terminal stage at which the controller either accepts a diagnosis or defers, and let $D_\pi(Z)=1$ if it returns an autonomous diagnosis and $D_\pi(Z)=0$ if it defers. Define
\begin{align}
  \risk(\pi)
  &= \Pr\!\left(\widehat{Y}_{T_\pi}\neq Y \mid D_\pi=1\right), \\
  \coverage(\pi)
  &= \Pr(D_\pi=1), \\
  \cost(\pi)
  &= \mathbb{E}\!\left[
       \sum_{t<T_\pi} c(A_{t+1})
       + c_{\mathrm{def}}(1-D_\pi)
     \right].
\end{align}
These quantities separate diagnostic reliability, autonomous participation, and resource use. Specifically, $\risk(\pi)$ is the diagnostic error rate among cases that the controller handles autonomously, rather than among all episodes. $\coverage(\pi)$ is the population fraction receiving an autonomous diagnosis; the remaining fraction is deferred. Finally, $\cost(\pi)$ is the expected cumulative cost of requested actions, plus a downstream-review penalty $c_{\mathrm{def}}$ whenever the case is deferred.

We seek the least costly policy in the candidate family $\Pi$ while requiring its selective diagnostic error to be at most $\alpha$ and its autonomous coverage to be at least $\gamma$:
\begin{equation}
  \min_{\pi\in\Pi}\ \cost(\pi)
  \quad \text{s.t.}\quad
  \risk(\pi)\leq\alpha,\qquad
  \coverage(\pi)\geq\gamma .
  \label{eq:objective}
\end{equation}
Thus, $\alpha$ specifies the maximum tolerated error rate among autonomous diagnoses, whereas $\gamma$ prevents the controller from achieving low error merely by deferring most cases. Because conditional risk is undefined when $\coverage(\pi)=0$, the positive coverage constraint also excludes the degenerate always-defer policy. The defer penalty affects policy design and cost comparisons, but it is not part of the subsequent binomial tests of risk and coverage.

\begin{figure}[t]
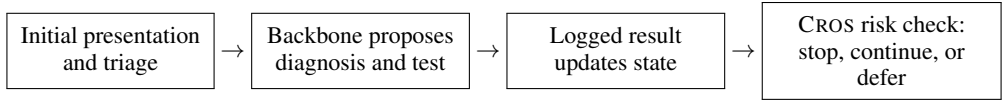

\centering
\setlength{\fboxsep}{5pt}
\small
\fbox{\parbox{0.17\linewidth}{\centering Initial presentation\\and triage}}
$\rightarrow$
\fbox{\parbox{0.18\linewidth}{\centering Backbone proposes\\diagnosis and test}}
$\rightarrow$
\fbox{\parbox{0.18\linewidth}{\centering Logged result\\updates state}}
$\rightarrow$
\fbox{\parbox{0.20\linewidth}{\centering \cros{} risk check:\\stop, continue, or defer}}
\caption{A sequential episode. The backbone generates the common forced-continuation trajectory and proposes each test; \cros{} controls whether to stop, continue with that proposal, or defer. It does not select the test identity.}
\label{fig:pipeline}
\end{figure}

\subsection{Risk-ranked stopping policies}

We train an auxiliary estimator $r_\theta(S_t)\in[0,1]$ for the event $\widehat{Y}_t\neq Y$. Features include the $K$ class probabilities, maximum probability, probability margin, entropy, stage, fraction of missing results, cumulative resource cost, latency, and the backbone's native stop score. Training uses episode-wise out-of-fold predictions so that multiple states from one patient never cross folds.

For a horizon $h$ and threshold $\tau$, the deterministic policy $\pi_{h,\tau}$ stops at the first $t\leq h$ satisfying $r_\theta(S_t)\leq\tau$ and otherwise defers at $h$. A disjoint selection set freezes the estimator, a finite candidate list $\Pi_0=\{\pi_1,\ldots,\pi_L\}$, candidate order, and all design hyperparameters before calibration labels are accessed. The learned ranker may be misspecified; validity below depends only on a fresh exchangeable calibration sample.

\subsection{Exact joint tests and multiplicity control}

After a candidate policy $\pi_j$ has been frozen, calibration asks whether it satisfies both population requirements: selective diagnostic error at most $\alpha$ and autonomous coverage at least $\gamma$. When $\pi_j$ is applied to $n$ calibration episodes, let
\begin{equation}
M_j=\sum_{i=1}^{n}D_{\pi_j}(Z_i),
\qquad
E_j=\sum_{i=1}^{n}
D_{\pi_j}(Z_i)
\mathbf{1}\!\left\{
\widehat{Y}_{T_{\pi_j},i}\neq Y_i
\right\}.
\end{equation}
Here, $M_j$ is the number of episodes receiving an autonomous diagnosis, while $E_j$ is the number of errors among those autonomous diagnoses. Thus, $M_j/n$ is the empirical autonomous coverage and, when $M_j>0$, $E_j/M_j$ is the empirical selective diagnostic error.

A candidate is invalid if either its population risk exceeds $\alpha$ or its population coverage falls below $\gamma$. We therefore test the union null
\begin{equation}
H_j:
\{\risk(\pi_j)>\alpha\}
\ \cup\
\{\coverage(\pi_j)<\gamma\}.
\end{equation}
Rejecting $H_j$ requires evidence against both failure modes: the policy must have sufficiently few autonomous errors and sufficiently many autonomous diagnoses.

Let $F_{\mathrm{Bin}}(k;m,p)$ denote the probability that a $\mathrm{Binomial}(m,p)$ random variable is at most $k$. The corresponding one-sided exact component $p$-values are
\begin{equation}
p_{R,j}
=
F_{\mathrm{Bin}}(E_j;M_j,\alpha),
\qquad
p_{C,j}
=
1-F_{\mathrm{Bin}}(M_j-1;n,\gamma).
\label{eq:pvals}
\end{equation}
The risk $p$-value becomes small when the policy makes unusually few errors relative to the boundary rate $\alpha$, whereas the coverage $p$-value becomes small when it autonomously diagnoses unusually many episodes relative to the boundary rate $\gamma$. We set $p_{R,j}=1$ when $M_j=0$, because an always-defer policy provides no evidence about conditional diagnostic risk. Because both constraints must be supported, the intersection-union test uses
\begin{equation}
p_j=\max(p_{R,j},p_{C,j}).
\label{eq:joint-p}
\end{equation}
This combined value is small only when both component $p$-values are small. For a single pre-frozen policy, no multiplicity adjustment is needed. When several candidates are tested, we use fixed-sequence LTT, Holm's step-down procedure \citep{holm1979simple}, or Bonferroni to control the probability of certifying any invalid candidate. In a prospective study, a rejected candidate is certified under Theorem~\ref{thm:certificate}; here, the same event is called an exploratory calibration pass because the labels are not untouched.

\begin{theorem}[Finite-sample joint control]
\label{thm:certificate}
Assume the calibration episodes are i.i.d. (or exchangeable with the future population), and the complete candidate policies and multiple-testing rule are fixed independently of calibration outcomes. Then each $p_j$ in Eq.~(\plaineqref{eq:joint-p}) is super-uniform under $H_j$.. If the testing rule controls family-wise error at $\delta$, the probability that any certified policy violates either $\risk(\pi)\leq\alpha$ or $\coverage(\pi)\geq\gamma$ is at most $\delta$.
\end{theorem}

The proof is in Appendix~\ref{app:proof}. The statement is finite-sample and makes no assumption that $r_\theta$ is calibrated.

\subsection{Randomized sparse mixtures}

A finite threshold-and-horizon grid may contain no single deterministic policy that achieves the desired risk--coverage trade-off at minimum cost. We therefore allow episode-wise randomization over the frozen candidate family $\Pi_0=\{\pi_1,\ldots,\pi_L\}$. Let $w_j$ be the probability of selecting policy $\pi_j$, with weights in the probability simplex $\Delta_L:=\{w\in\mathbb{R}_+^L:\sum_{j=1}^{L}w_j=1\}$. The resulting randomized policy $\pi_w$ independently samples $J\sim\operatorname{Categorical}(w)$ at the beginning of each episode and applies $\pi_J$ throughout that episode.

On the selection split, let $\widehat{\cost}_j$ be the empirical mean cost of $\pi_j$, $\widehat c_j=\widehat{\Pr}(D_{\pi_j}=1)$ its empirical autonomous coverage, and $\widehat q_j=\widehat{\Pr}(D_{\pi_j}=1,\widehat Y\neq Y)$ its empirical error mass. When $\widehat c_j>0$, its empirical selective diagnostic error is $\widehat q_j/\widehat c_j$. Because the component is sampled independently for each episode, the mixture's expected cost, error mass, and coverage are $\sum_jw_j\widehat{\cost}_j$, $\sum_jw_j\widehat q_j$, and $\sum_jw_j\widehat c_j$, respectively. We choose the least costly mixture satisfying the selection-stage risk and coverage targets:
\begin{equation}
\begin{aligned}
\min_{w\in\Delta_L}\quad
&\sum_{j=1}^{L}w_j\widehat{\cost}_j \quad
\text{s.t.}\quad
&\sum_{j=1}^{L}w_j
\left(
\widehat q_j-\alpha_{\mathrm{des}}\widehat c_j
\right)
\leq0,
\qquad
\sum_{j=1}^{L}w_j\widehat c_j
\geq\gamma_{\mathrm{des}}.
\end{aligned}
\label{eq:lp}
\end{equation}
The first constraint is equivalent to requiring the mixture's empirical selective diagnostic error, $(\sum_jw_j\widehat q_j)/(\sum_jw_j\widehat c_j)$, to be at most $\alpha_{\mathrm{des}}$; the denominator is positive because the second constraint requires coverage of at least $\gamma_{\mathrm{des}}>0$.

We use the stricter selection-design targets $(\alpha_{\mathrm{des}},\gamma_{\mathrm{des}})=(0.20,0.80)$, while the subsequent calibration tests use $(\alpha,\gamma)=(0.25,0.70)$. These design margins provide a buffer against selection-sample variation but are not themselves a statistical certificate. \begin{proposition}[Sparsity and validity]
\label{prop:sparse-validity}
If the linear program in Eq.~(\plaineqref{eq:lp}) is feasible, it admits an optimal basic feasible solution in which at most three weights $w_j$ are positive. If the mixture weights and episode-wise randomization mechanism are frozen before calibration, the induced randomized controller is a single frozen policy whose risk and coverage can be tested under Theorem~\ref{thm:certificate}, subject to the theorem's sampling assumptions.
\end{proposition}

Proposition~\ref{prop:sparse-validity} implies that, although the optimization considers $L$ deterministic candidates, an optimal mixture uses at most three of them. This bound follows from the simplex equality and the two risk and coverage design constraints in Eq.~(\plaineqref{eq:lp}). At deployment, component sampling is independent across episodes and is never conditioned on patient characteristics or intermediate observations. Equivalently, the random seed can be treated as part of each i.i.d.\ episode. A single component draw reused for all calibration episodes would instead introduce shared randomness and would not justify the same exact binomial test. A proof is provided in Appendix~\ref{app:proof}.

\section{Experimental Design}

\subsection{Benchmark and study splits}

We construct a retrospective sequential-diagnosis benchmark by linking credentialed-access MIMIC-IV-ED v2.2 \citep{johnson2023mimiced} with MIMIC-IV-Ext-CDS v1.0.2 \citep{gaber2025extcds}. Retaining the earliest eligible ED stay per patient yields 1,834 patient-level episodes across nine abdominal-pain diagnosis classes. The initial state includes the deidentified history of present illness, chief complaint, demographics, and triage measurements.

The environment exposes 12 action groups spanning repeated vital signs, laboratory studies, electrocardiography, imaging, and microbiology. Continuing along an episode reveals the recorded result of the test proposed by the backbone. Because most discharge-note test snippets lack reliable acquisition timestamps, this environment is a logged retrospective benchmark rather than a causal simulator of alternative testing decisions. An unavailable result is represented by \textsf{NO\_RECORDED\_RESULT} and retains its prespecified cost. The complete action definitions and cost vector are reported in Appendix~\ref{app:benchmark}.

Patients are divided into 1,100 development, 367 calibration, and 367 evaluation episodes, with no patient overlap. The development cohort is further separated into 935 episodes for backbone and risk-ranker fitting and 165 episodes for policy selection. All model fitting, candidate construction, policy ordering, and mixture optimization use only these development partitions. The primary population targets are selective diagnostic error $\alpha=0.25$, autonomous coverage $\gamma=0.70$, and family-wise error level $\delta=0.05$.

\paragraph{Exploratory status.}
Calibration and evaluation labels had been accessed during earlier method development. Consequently, all reported calibration passes, $p$-values, and confidence bounds are interpreted descriptively rather than as realized prospective certificates. The current freeze prevents further outcome-dependent modification but cannot restore statistical independence; applying Theorem~\ref{thm:certificate} confirmatorily requires a new untouched cohort.

\subsection{Model instantiation and comparators}

The diagnostic backbone is Qwen2.5-7B-Instruct \citep{qwen2024qwen25}, adapted with LoRA \citep{hu2022lora} using an official-code-derived LA-CDM training pipeline. Training uses development data only. Because our implementation transfers the official architecture and training structure to a fixed common-path environment, it should not be interpreted as a prompt-equivalent reproduction of LA-CDM's original free-form rollouts. Model configuration, optimization, prompts, seeds, and provenance checks are provided in Appendix~\ref{app:protocol}.

A cross-fitted histogram gradient-boosting ranker estimates state-level diagnostic error from backbone probabilities, uncertainty summaries, trajectory state, accumulated cost, missingness, and the native stopping score. Crossing 10 target coverages with 13 policy horizons produces 130 threshold-and-horizon candidates. The disjoint policy-selection split freezes a tested family of 12 deterministic policies and, separately, one optimized randomized mixture. Candidate construction, ranker hyperparameters, and threshold estimation are detailed in Appendix~\ref{app:protocol}.

Every stopping method receives the same backbone probabilities, diagnoses, proposed actions, and forced-continuation trajectory. This common-path protocol holds diagnosis and acquisition behavior fixed, isolating the decision to stop, continue, or defer. Comparators include HPI-only and full-workup endpoints, fixed-stage and confidence-threshold stopping, the native LA-CDM-style rule with and without confidence deferral, and empirical cost minimization. We additionally evaluate calibration and multiplicity procedures, score and history ablations, and a mixture-weight ablation. These labels describe experimental roles rather than additional \cros{} methods. Complete comparator definitions are provided in Appendix~\ref{app:protocol}.
\paragraph{Method nomenclature and policy construction.}
\cros{} is the overall risk-constrained stopping framework and is instantiated here through exactly two named controllers. \crosdet{} is a deterministic threshold-and-horizon policy selected from the locked development family, and \crosmix{} is an episode-wise randomized, LP-optimized mixture over that same frozen deterministic family. All other experimental labels denote fixed-information, confidence-based, native-agent, or empirical baselines; calibration and multiplicity procedures; ranker ablations; a policy-design ablation; or evaluation modes of \crosmix{}. They are not additional \cros{} methods.

Specifically, \crosdet{} is the lowest-cost member of the full-ranker 130-policy grid that satisfies the locked selection-split design margins, risk at most 20\% and coverage at least 80\%; in the present run it is $(h=3,\tau=0.3286)$. The best deterministic component within the optimized mixture support is identical to \crosdet{} in this run and is therefore not displayed separately. \crosmix{} always denotes the same three support policies and the same nonnegative LP-optimized weights, fitted using only the policy-selection split to minimize expected cost subject to the 20\%/80\% design margins. \crosmix{}, analytic expectation integrates component contributions for each episode, whereas \crosmix{}, realized draw uses one frozen independent episode-wise component draw and yields integer autonomous/error counts for exact binomial testing. These are two evaluation modes of one controller. The \uniformmixture{} uses the same frozen support with equal rather than optimized weights and is a policy-design ablation; it is not guaranteed to satisfy the selection margins.

The \singlecandidate{}, \fixedsequenceltt{}, \holmtesting{}, and \bonferronitesting{} are calibration or multiplicity procedures applied to the same 12 selection-frozen deterministic candidates. They change the testing and return rule, not the risk ranker, backbone, or basic threshold-and-horizon policy. In this run, \holmtesting{} returns the same deterministic controller as \fixedsequenceltt{}, while \bonferronitesting{} returns a different deterministic controller. Empirical ERM instead minimizes selection-split cost without requiring the joint risk--coverage criterion.

Myopic value-of-information and free-form multi-agent systems are not included in the primary paired comparison because they select different actions and therefore induce different trajectories. Appendix~\ref{app:protocol} specifies how such systems would enter a future end-to-end confirmatory study.

\subsection{Evaluation}

The primary outcomes are selective diagnostic error, autonomous coverage, total relative resource cost, and number of requested tests. We also report error mass, $\Pr(D=1,\widehat Y\neq Y)$, which measures the population fraction receiving an incorrect autonomous diagnosis and therefore does not decrease merely because the controller defers additional cases.

Uncertainty estimates respect the patient-level sampling unit. We report exact one-sided Clopper--Pearson bounds for risk and coverage and paired patient-level bootstrap intervals for method contrasts. Randomized mixtures are evaluated both by analytically averaging component contributions and by repeated episode-wise realizations to assess randomization stability. Prespecified sensitivity analyses examine deferral penalties, action costs, recorded-result missingness, diagnosis-language masking, and development-split stability. Full metric definitions, resampling procedures, and sensitivity settings are reported in Appendices~\ref{app:benchmark}--\ref{app:protocol}.

\section{Results}

\subsection{Exploratory joint calibration}

Table~\ref{tab:cros} reports realized calibration outcomes and
descriptive evaluation performance. All returned policies have
evaluation point estimates below 25\% risk and above 70\% coverage.
The prespecified \crosmix{} realization is the least costly displayed
policy, attaining 16.3\% error at 78.5\% coverage with cost 5.68 and
0.68 requested actions.

The \singlecandidate{} and the \crosmix{} realized draw are each
tested once, \fixedsequenceltt{} follows its frozen order, and
\bonferronitesting{} tests 12 candidates; \holmtesting{} returns the
same deterministic controller as \fixedsequenceltt{} in this run.
Appendix~\ref{app:exact} gives the component
tests, exact bounds, weights, and seed audit. Later mixture comparisons
use the \crosmix{} analytic expectation (0.169 risk, 0.788 coverage, 5.57 cost,
and 0.679 actions), rather than the realization in
Table~\ref{tab:cros}.

\begin{table}[H]
\caption{Exploratory calibration and evaluation results
($n_{\mathrm{cal}}=n_{\mathrm{eval}}=367$). ``Auto/err'' denotes
autonomous diagnoses/errors. The table reports raw joint $p$-values
before multiplicity adjustment. The first three rows are testing
procedures applied to the same frozen deterministic family; the final
row is one realized draw of \crosmix{}. Holm testing returns the same
controller as fixed-sequence LTT and is not duplicated.}
\label{tab:cros}
\centering
\scriptsize
\setlength{\tabcolsep}{3.2pt}
\resizebox{\linewidth}{!}{%
\begin{tabular}{lrrrrrrr}
\toprule
Calibration configuration / controller
& Cal auto/err
& Cal risk
& Cal cov.
& Raw joint $p$
& Eval risk
& Eval cov.
& Cost / tests \\
\midrule
\singlecandidate{}
& 282/48 & .170 & .768 & .00210
& .166 & .758 & 17.14 / 3.27 \\

\fixedsequenceltt{}
& 275/43 & .156 & .749 & .02116
& .165 & .760 & 12.43 / 2.09 \\

\bonferronitesting{}
& 288/48 & .167 & .785 & .000442
& .168 & .779 & 13.88 / 2.52 \\

\crosmix{}, realized draw
& 280/51 & .182 & .763 & .00435
& \textbf{.163} & \textbf{.785}
& \textbf{5.68 / .68} \\
\bottomrule
\end{tabular}}
\end{table}

\subsection{Risk-ranker and stopping ablation}

The full ranker achieves state-error AUROC 0.853, compared with 0.715
for the maximum-probability ranker and 0.552 for the
\nativescoreranker{}
(Table~\ref{tab:ranker_ablation}). Among these alternatives, only the
full-ranker controller satisfies the locked selection margins and has
an exploratory calibration pass. The maximum-probability ranker is
cheaper but lacks comparable calibration evidence. The
\nohistoryranker{} nearly matches the full
ranker (AUROC 0.852), so this experiment does not isolate a material
benefit from stage, cost, latency, and missing-history features. The
supported contribution is therefore improved risk ranking and
calibration power, not unconditional cost dominance.

\begin{table}[H]
\caption{Risk-ranker ablation. ``Sel.'' indicates the locked
selection margins. AUROC intervals use 10,000 patient-level resamples;
calibration $p$-values are descriptive.}
\label{tab:ranker_ablation}
\centering
\scriptsize
\setlength{\tabcolsep}{3.1pt}
\begin{tabular}{lclrrrr}
\toprule
Ranker/controller configuration
& Sel.
& State AUROC [95\% CI]
& Cal. $p$
& Eval risk
& Cov.
& Cost / tests \\
\midrule
\crosdet{}, full ranker
& yes & .853 [.820,.885] & .005 & .175 & .809 & 6.47 / .87 \\

\nohistoryranker{}
& yes & .852 [.820,.884] & .040 & .190 & .817 & 6.06 / .83 \\

Entropy-margin ranker
& no & .714 [.678,.750] & .805 & .204 & .828 & 1.72 / .00 \\

Maximum-probability ranker
& no & .715 [.679,.751] & .504 & .185 & .782 & 2.18 / .00 \\

\nativescoreranker{}
& no & .552 [.533,.572] & .902 & .250 & .719 & 31.58 / 6.18 \\
\bottomrule
\end{tabular}
\end{table}

\subsection{Policy and matched comparisons}

\crosmix{}, analytic expectation is 0.903 cost units cheaper than \crosdet{}
(95\% CI $[-1.135,-0.682]$) and requests 0.195 fewer actions; its risk
difference is unresolved and its coverage is 0.021 lower
(Table~\ref{tab:mixture_ablation}). The \uniformmixture{} is another 0.977
units cheaper on evaluation, but misses both locked selection margins.
Thus, optimization enforces the development constraints rather than
guaranteeing the lowest future cost.

Compared with confidence thresholding and native stopping with
deferral, \crosmix{}, analytic expectation reduces risk by 0.088 and
0.063 and cost by 12.54
and 4.84 units, respectively; coverage differences are unresolved.
ERM is cheaper but has higher risk and fails exploratory calibration.
Appendix~\ref{app:paired} reports the complete paired intervals.
Cost advantages persist in the primary, imaging-sensitive, and
missing-result-sensitive scenarios, but not under every deferral
penalty (Appendix~\ref{app:cost_details}).

\begin{table}[H]
\caption{Policy-design ablation using analytic mixture expectations.
``Sel.'' indicates the frozen 20\% risk and 80\% coverage design
constraints, not an evaluation guarantee. The fixed-sequence LTT row
reports the deterministic controller returned by that testing
procedure.}
\label{tab:mixture_ablation}
\centering
\small
\setlength{\tabcolsep}{4.1pt}
\begin{tabular}{lcrrrrr}
\toprule
Policy configuration & Sel. & Risk & Cov. & Error mass & Cost & Tests \\
\midrule
\crosdet{}
& yes & .175 & .809 & .142 & 6.47 & .87 \\

\uniformmixture{}
& no & .172 & .786 & .135 & 4.59 & .50 \\

\crosmix{}, analytic expectation
& yes & .169 & .788 & .133 & 5.57 & .68 \\

\fixedsequenceltt{}
& -- & .165 & .760 & .125 & 12.43 & 2.09 \\

Empirical ERM
& -- & .202 & .850 & .172 & 1.50 & .00 \\
\bottomrule
\end{tabular}
\end{table}

Against the confidence-threshold controller, \crosmix{}, analytic expectation
reduces selective risk by 0.088 and cost by 12.54 units; against native
stopping with deferral, it reduces risk by 0.063 and cost by 4.84
units, with unresolved coverage differences in both comparisons.
ERM remains less costly but has higher risk and fails exploratory
calibration. Complete patient-paired intervals and error-mass
comparisons are reported in Appendix~\ref{app:paired}.

The cost advantage over the principal calibrated and
confidence-based comparators persists under the primary,
imaging-sensitive, and missing-result-sensitive cost scenarios, but
not under every deferral penalty. Appendix~\ref{app:cost_details}
reports the complete frozen-decision sensitivity analysis.

\subsection{Trajectory and robustness audits}

Forced-continuation error is 28.3\% with HPI alone, 27.5\% after one action, 35.7\% at stage 8, and 34.3\% after full workup, while cost and recorded-result missingness increase along the trajectory (Appendix Figure~\ref{fig:stagewise}). Thus, this backbone does not show monotone gains from additional acquisition. Diagnosis-language masking lowers ranker AUROC by approximately 0.014 and modestly reduces coverage and increases cost; risk changes are unresolved (Appendix Table~\ref{tab:mask_full}). Action availability is also highly structured and may encode clinicians' historical ordering. Across 20 development resplits, \crosdet{} nominally passes in 15, whereas \crosmix{} is feasible in 15 and passes in only six. These dependent audits demonstrate design sensitivity, not independent confirmation. Class-level failures further preclude any subgroup safety claim (Appendix~\ref{app:subgroups}).

\section{Discussion and Limitations}
\label{sec:limitations}

The results support a prospectively frozen follow-up, not deployment
or a clinical safety claim. The full ranker outperforms simple scores,
joint tests have apparent power for several frozen policies, and
\crosmix{} reduces cost relative to \crosdet{}. However, the
\nohistoryranker{} nearly matches the full ranker, the
\uniformmixture{} is cheaper despite
missing the selection margins, forced continuation is non-monotone,
and mixture feasibility varies across development splits.

The claims are limited by previously accessed labels; a single-center,
note-derived benchmark without reliable non-vital timestamps or
counterfactual outcomes; informative missingness; unadjudicated labels,
proxy misses, and costs; marginal rather than subgroup control; one
backbone-training seed; and dependent resplit audits. Common-path
evaluation isolates stopping but does not measure end-to-end gains when
agents choose different actions.

A confirmatory study should lock a never-viewed temporal or external
cohort, prespecify labels, costs, subgroup hypotheses, sample size, and
one testing graph, and obtain blinded clinical adjudication. It should
evaluate both common-path stopping and end-to-end acquisition without
changes after calibration access.

\section{Conclusion}

\cros{} makes sequential diagnostic stopping auditable through joint
testing of selective error and autonomous coverage. The current
retrospective results support feasibility and a lower-cost
deterministic--mixture trade-off, but label reuse, structured
missingness, resplit instability, and subgroup failures require a new
prospectively locked study before any safety claim.


\bibliographystyle{plainnat}
\bibliography{references}

\appendix

\section{Proofs}
\label{app:proof}

\begin{proof}[Proof of Theorem~\ref{thm:certificate}]
Fix a candidate $\pi_j$ independently of the calibration outcomes. Conditional on $M_j=m>0$, the number of autonomous errors is $E_j\sim\mathrm{Binomial}(m,r_j)$ under i.i.d. sampling, where $r_j=\risk(\pi_j)$. For the null $r_j>\alpha$, the lower-tail statistic $F_{\mathrm{Bin}}(E_j;m,\alpha)$ is largest at the boundary in the rejection-relevant direction; its discreteness makes it super-uniform. Thus $p_{R,j}$ is valid for $H_{R,j}:r_j>\alpha$. Setting $p_{R,j}=1$ when $m=0$ preserves validity.

Marginally, $M_j\sim\mathrm{Binomial}(n,c_j)$ with $c_j=\coverage(\pi_j)$. Under $H_{C,j}:c_j<\gamma$, the upper-tail value $1-F_{\mathrm{Bin}}(M_j-1;n,\gamma)$ is super-uniform, again with the boundary least favorable in the rejection direction. The candidate null is the union $H_j=H_{R,j}\cup H_{C,j}$. The intersection-union test rejects only if both component tests reject. Therefore, for any distribution in $H_j$, at least one component null is true and
\[
\Pr\{\max(p_{R,j},p_{C,j})\leq u\}
\leq
\Pr\{p_{k,j}\leq u\}\leq u,
\]
where $k$ indexes a true component null. Hence $p_j=\max(p_{R,j},p_{C,j})$ is super-uniform. Applying any valid family-wise-error procedure at level $\delta$ to the finite, pre-frozen family implies that the probability of rejecting at least one true $H_j$ is at most $\delta$. A certified policy violates a desired constraint exactly when its union null is true, proving the result.
\end{proof}

\begin{proof}[Proof of the sparsity and randomized-policy proposition]
Write the linear program in standard form after adding slack variables. The policy-weight vector obeys one simplex equality. At a nondegenerate extreme point, at most two independent design inequalities can be active in addition to the simplex equality. Therefore at most three policy weights need be basic and positive; a degenerate optimum has no larger support, and an optimal basic feasible solution can always be selected.

For validity, augment each episode with $U_i\sim\mathrm{Uniform}(0,1)$, independently across episodes and independent of the clinical variables. The frozen mixture maps $U_i$ to a deterministic component and then applies it to $Z_i$. Thus $(Z_i,U_i)$ are i.i.d. and the mixture is a single fixed randomized policy. Its autonomous indicator and error indicator satisfy the same binomial conditioning argument as in Theorem~\ref{thm:certificate}. The conclusion fails if mixture weights are estimated on calibration outcomes or if one shared random component is drawn for the entire calibration sample.
\end{proof}

\section{Benchmark, Costs, and Ranker Configuration}
\label{app:benchmark}

\paragraph{Cohort construction.}
We link credentialed-access MIMIC-IV-ED v2.2
\citep{johnson2023mimiced} with MIMIC-IV-Ext-CDS v1.0.2
\citep{gaber2025extcds}. Cohort construction retains the earliest
qualifying emergency-department stay for each subject before any data
partitioning, so each patient contributes exactly one episode. The
resulting benchmark contains 1,834 episodes assigned by the frozen
label-mapping code to nine abdominal-pain diagnosis classes:
appendicitis, biliary disease, bowel obstruction, diverticulitis,
gastroenteritis/colitis, nonspecific abdominal pain, pancreatitis,
renal colic, and urinary infection.

The initial presentation includes a deidentified history of present
illness, chief complaint, demographic variables, and triage
measurements. These fields constitute the stage-zero information
available before any action is requested. The executable cohort query,
label mapping, and episode identifiers are versioned as part of the
frozen benchmark artifacts.

\paragraph{Patient-level partitions.}
Patients are partitioned into 1,100 development, 367 calibration, and
367 evaluation episodes. The development set is further divided into
935 ranker-training and 165 policy-selection episodes. All states from
one episode remain in the same partition and, during ranker
cross-fitting, in the same fold. The four resulting patient sets are
therefore disjoint. The 935-episode subset is used to fit the
state-error ranker, whereas thresholds, deterministic candidates,
candidate order, and mixture weights are designed only on the
165-episode selection subset. The calibration and evaluation subsets
are used for the analyses described in the main text. As discussed in
Section~\ref{sec:limitations}, prior access to their labels makes the
present results exploratory rather than confirmatory.

\begin{table}[H]
\caption{Benchmark summary. Here, $\delta$ is the family-wise error
level used by the joint testing procedure.}
\label{tab:benchmark_summary}
\centering
\begin{tabular}{ll}
\toprule
Item & Value \\
\midrule
Unique patients / episodes & 1,834 / 1,834 \\
Development / calibration / evaluation & 1,100 / 367 / 367 \\
Ranker train / policy selection & 935 / 165 \\
Diagnosis classes & 9 \\
Action groups & 12 \\
Maximum action horizon $H$ & 12 \\
Full-workup relative cost & 51.9 \\
Deferral penalty & 10.0 \\
Selective-risk target $\alpha$ & .25 \\
Coverage target $\gamma$ & .70 \\
Family-wise error level $\delta$ & .05 \\
\bottomrule
\end{tabular}
\end{table}

\paragraph{Retrospective action environment.}
The action space comprises rechecking vital signs, complete blood
count, metabolic panel, hepatic panel, lipase, urinalysis,
electrocardiogram, X-ray, ultrasound, computed tomography, magnetic
resonance imaging, and microbiology. The maximum horizon $H=12$
specifies the largest number of action opportunities considered in a
forced-continuation episode; it does not imply that all 12 results are
available in the record.

For each action group, recorded findings are extracted from
discharge-note test snippets. Except for vital signs, these snippets
generally lack acquisition timestamps sufficiently reliable to
reconstruct the clinical order in which results became available.
Consequently, an observation represents information found in the
retrospective record for the requested action group, rather than the
counterfactual result of ordering that action prospectively. When no
corresponding result is available, the environment returns
\textsf{NO\_RECORDED\_RESULT}. This missing-result token is retained as
an observed outcome and may affect subsequent backbone predictions.

The benchmark should therefore be interpreted as a logged-acquisition
environment for comparing stopping rules on a shared backbone
trajectory. It does not estimate how ordering a different test would
alter subsequent care, physiology, documentation, or diagnostic
outcomes. In particular, \cros{} controls whether the next
backbone-proposed action is executed, but it does not replace that
action or simulate an alternative trajectory.

\begin{table}[H]
\caption{Frozen relative action-cost vector shared by every controller.
The full-workup total is the sum of the 12 action costs; the deferral
penalty is separate.}
\label{tab:action_costs}
\centering
\small
\setlength{\tabcolsep}{7pt}
\begin{tabular}{lrlr}
\toprule
Action group & Cost & Action group & Cost \\
\midrule
Recheck vitals & 0.2 & Electrocardiogram & 1.5 \\
Complete blood count & 1.0 & X-ray & 4.0 \\
Metabolic panel & 1.2 & Ultrasound & 6.0 \\
Hepatic panel & 1.2 & Computed tomography & 12.0 \\
Lipase & 1.0 & Magnetic resonance imaging & 20.0 \\
Urinalysis & 0.8 & Microbiology & 3.0 \\
\midrule
\multicolumn{3}{r}{All 12 actions} & 51.9 \\
\multicolumn{3}{r}{Deferral penalty} & 10.0 \\
\bottomrule
\end{tabular}
\end{table}

\paragraph{Cost interpretation.}
An attempted action incurs its Table~\ref{tab:action_costs} cost even
when the retrospective record returns
\textsf{NO\_RECORDED\_RESULT}. The reported cost is a normalized
resource index intended to support controlled comparisons among
stopping policies. It is not a hospital bill, reimbursement amount,
radiation dose, patient utility, or estimate of clinical harm.
Likewise, the deferral penalty represents a stylized downstream-review
cost rather than a measured clinical or monetary quantity.

The frozen-decision sensitivity analysis varies the deferral penalty
over $\{5,10,15,20,25,30\}$ and separately applies multipliers
$\{0.5,1,2\}$ to laboratory, imaging, and attempted-but-missing action
costs according to the frozen category mapping. These analyses retain
the originally frozen stopping decisions and mixture weights; they
recalculate costs but do not refit the ranker, reconstruct thresholds,
or reoptimize the policies.

\paragraph{Risk-ranker inputs.}
Each stage-level ranker target is the binary indicator
$\mathbf{1}\{\widehat{Y}_t\neq Y\}$, which records whether the
backbone's current diagnosis is incorrect. The ranker uses the
backbone's $K$ diagnosis probabilities together with their maximum,
the margin between the two largest probabilities, predictive entropy,
the current stage, the fraction of previously requested actions with
no recorded result, cumulative relative cost, the recorded latency
feature, and the backbone's native stopping score. It does not select
the next action or directly modify the backbone diagnosis.

The 935 ranker-training episodes yield 12,155 stage-level records.
Because states from the same patient are correlated, cross-fitting and
all resampling operations are performed at the episode level rather
than at the state level.

\begin{table}[H]
\caption{Frozen histogram gradient-boosting risk-ranker configuration.
Automatic early stopping follows scikit-learn's implementation.}
\label{tab:ranker_hyperparameters}
\centering
\small
\setlength{\tabcolsep}{6pt}
\begin{tabular}{p{0.34\linewidth}p{0.57\linewidth}}
\toprule
Hyperparameter & Frozen value \\
\midrule
Estimator / loss
& \texttt{HistGradientBoostingClassifier} / log loss \\
Learning rate
& 0.05 \\
Maximum boosting iterations
& 160 (one tree per iteration for binary classification) \\
Maximum leaves / depth
& 15 / no explicit depth cap \\
Minimum samples per leaf
& 50 \\
$L_2$ regularization
& 2.0 \\
Early stopping
& \texttt{auto}; validation fraction 0.10 \\
Patience / tolerance
& 10 iterations / $10^{-7}$ \\
Outer cross-fitting
& Five episode-wise stratified folds, shuffled \\
Random seeds
& Splitter base 20260902; fold fits 20260902--20260906;
final fit 20261002 \\
Realized iterations
& 160 in each outer-fold fit; 149 in the final
12,155-state fit \\
\bottomrule
\end{tabular}
\end{table}

\paragraph{Cross-fitting and final ranker.}
The 935 ranker-training episodes are divided into five stratified
outer folds. For each fold, the estimator is fitted using the other
four folds and generates predictions for all states belonging to the
held-out episodes. This construction prevents states from the same
patient from appearing on both sides of an outer-fold fit. The
resulting out-of-fold scores are used for ranker-development
diagnostics.

Each outer-fold training set contains fewer than 10,000 state records,
so scikit-learn's automatic early-stopping condition is not activated
and all five estimators reach the 160-iteration limit. After
cross-fitting, the final ranker is refitted on all 12,155 state
records. Because this fit exceeds the automatic early-stopping sample
threshold, it uses the internal 10\% validation fraction and stops
after 149 iterations. This final frozen estimator produces the risk
scores used on the disjoint policy-selection, calibration, and
evaluation episodes.

\paragraph{Coverage-indexed threshold construction.}
Thresholds are constructed only on the 165-episode policy-selection
split. For horizon $h\in\{0,\ldots,12\}$ and selection episode $i$,
define the best-so-far risk score and the relevant one-indexed order
statistic as
\[
m_i(h)=\min_{0\leq t\leq h}r_\theta(S_{it}),
\qquad
k_q=\left\lceil q(165-1)\right\rceil+1.
\]
For each target
$q\in\{0.72,0.75,0.78,0.80,0.82,0.85,0.88,0.90,0.93,0.95\}$,
the 165 values $m_i(h)$ are sorted and $\tau_{h,q}$ is set to their
$k_q$th order statistic. This is equivalent to
\texttt{numpy.quantile(..., method="higher")}. The ten $(q,k_q)$
pairs are $(.72,120)$, $(.75,124)$, $(.78,129)$, $(.80,133)$,
$(.82,136)$, $(.85,141)$, $(.88,146)$, $(.90,149)$,
$(.93,154)$, and $(.95,157)$.

These targets are coverage indices used to construct the grid, not
claims about calibration or population coverage. In particular, they
do not define ten global score cutoffs: every pair $(h,q)$ has its own
horizon-specific threshold. Policy $(h,q)$ stops at the first stage
$t\leq h$ satisfying
$r_\theta(S_{it})\leq\tau_{h,q}$ and otherwise defers at $h$. By
construction, its empirical autonomous coverage on the selection split
is at least $q$, with possible excess coverage when risk scores are
tied.

Crossing 13 horizons with 10 coverage indices yields 130 deterministic
candidates. Every threshold is then reused unchanged on calibration
and evaluation episodes. Using only the selection split, the frozen
pipeline additionally determines the tested 12-policy family, its
testing order, and the randomized-mixture weights. Candidate and
mixture design use
$(\alpha_{\mathrm{des}},\gamma_{\mathrm{des}})=(0.20,0.80)$,
whereas the subsequent joint tests use
$(\alpha,\gamma,\delta)=(0.25,0.70,0.05)$.

\section{Detailed Baseline and Masking Results}
\label{app:detailed_results}

\begin{table}[H]
\caption{Calibration and evaluation results for common-path
controllers. Joint $p$ is the raw intersection--union value;
\crosmix{}, realized draw is shown using the prespecified frozen
episode-wise assignment, while
analytic expectations are used in the paired tables.}
\label{tab:baseline_full}
\centering
\small
\setlength{\tabcolsep}{3.5pt}
\begin{tabular}{lrrrrrr}
\toprule
Controller / baseline
& Cal joint $p$
& Eval risk
& Eval cov.
& Cost
& Tests
& Error mass \\
\midrule
\crosmix{}, realized draw
& .00435 & .1632 & .7847 & 5.679 & .679 & .1281 \\
\fixedsequenceltt{}
& .02116 & .1649 & .7602 & 12.426 & 2.087 & .1253 \\
HPI-only
& .99820 & .2834 & 1.0000 & 0.000 & 0.000 & .2834 \\
Full workup
& .99999 & .3433 & 1.0000 & 51.900 & 12.000 & .3433 \\
LA-CDM native
& .99744 & .3079 & 1.0000 & 8.141 & 1.529 & .3079 \\
Fixed stage $h=8$
& .99999 & .3569 & 1.0000 & 36.023 & 8.000 & .3569 \\
Confidence threshold
& .79980 & .2568 & .7956 & 18.109 & 3.635 & .2044 \\
LA native + defer
& .72070 & .2324 & .7738 & 10.403 & 1.529 & .1798 \\
Empirical ERM
& .67940 & .2019 & .8501 & 1.499 & 0.000 & .1717 \\
\bottomrule
\end{tabular}
\end{table}

All controllers in Table~\ref{tab:baseline_full} receive the same
backbone diagnoses, action proposals, native stopping scores, and
forced-continuation observations. Their differences therefore reflect
stopping and deferral rather than alternative test acquisition. The
baseline calibration $p$-values are descriptive and do not imply that
each baseline belongs to the frozen multiplicity-controlled
\cros{} family. ERM is inexpensive because its selected policy
requests no additional action on evaluation, but it fails the joint
calibration criterion.

\paragraph{Diagnosis-language masking.}
An aggregate regex screen identified explicit diagnostic and
future-information language in a subset of the initial presentations.
We therefore froze two ontology-wide, label-independent masking
conditions before rerunning the backbone. The diagnosis-name condition
removes prespecified names and synonyms for all nine diagnosis classes.
The expanded condition additionally removes prespecified diagnostic
and future-information phrases. These conditions match 79/367 and
116/367 evaluation HPIs, respectively.

No backbone parameter, ranker, threshold, horizon, candidate order,
mixture weight, cost, or stopping policy is tuned on the masked
calibration or evaluation outcomes. Every controller within a masking
condition receives the same condition-specific common path.

\begin{table}[H]
\caption{Complete frozen evaluation masking outputs. Every controller
uses the same condition-specific common path. \crosmix{} is reported
in its analytic-expectation evaluation mode.}
\label{tab:mask_full}
\centering
\scriptsize
\setlength{\tabcolsep}{3.0pt}
\resizebox{\linewidth}{!}{%
\begin{tabular}{llrrrrr}
\toprule
Input & Controller / baseline & Selective risk & Coverage & Error mass & Cost & Tests \\
\midrule
Original
& HPI-only & .283 & 1.000 & .283 & .000 & .000 \\
& Full workup & .343 & 1.000 & .343 & 51.900 & 12.000 \\
& LA-CDM native & .308 & 1.000 & .308 & 8.141 & 1.529 \\
& Native + defer & .232 & .774 & .180 & 10.403 & 1.529 \\
& \crosdet{} & .175 & .809 & .142 & 6.469 & .875 \\
& \crosmix{}, analytic expectation & .169 & .788 & .133 & 5.566 & .679 \\
\midrule
Diagnosis names
& HPI-only & .297 & 1.000 & .297 & .000 & .000 \\
& Full workup & .365 & 1.000 & .365 & 51.900 & 12.000 \\
& LA-CDM native & .311 & 1.000 & .311 & 7.981 & 1.520 \\
& Native + defer & .253 & .777 & .196 & 10.215 & 1.520 \\
& \crosdet{} & .192 & .793 & .153 & 7.114 & .962 \\
& \crosmix{}, analytic expectation & .184 & .770 & .142 & 6.069 & .745 \\
\midrule
Names + future
& HPI-only & .300 & 1.000 & .300 & .000 & .000 \\
& Full workup & .371 & 1.000 & .371 & 51.900 & 12.000 \\
& LA-CDM native & .313 & 1.000 & .313 & 7.962 & 1.510 \\
& Native + defer & .252 & .790 & .199 & 10.060 & 1.510 \\
& \crosdet{} & .190 & .790 & .150 & 7.274 & .981 \\
& \crosmix{}, analytic expectation & .182 & .768 & .140 & 6.198 & .757 \\
\bottomrule
\end{tabular}}
\end{table}

Relative to the original inputs, diagnosis-name and expanded masking
change \crosmix{} analytic-expectation selective risk by $+.015$ (95\% CI
$[-.001,.033]$) and $+.013$ $[-.004,.032]$, coverage by $-.018$
$[-.036,-.0004]$ and $-.020$ $[-.040,-.001]$, and cost by $+.504$
$[.066,.968]$ and $+.632$ $[.150,1.130]$, respectively. The
corresponding ranker-AUROC changes are $-.014$
$[-.026,-.005]$ and $-.014$ $[-.026,-.003]$. Full-workup error
changes by $+.022$ $[.003,.044]$ and $+.027$ $[.005,.049]$.

The evaluation ranker AUROCs are .853, .839, and .840 for the
original, diagnosis-name, and expanded masking conditions;
the corresponding calibration AUROCs are .840, .842, and .842.
All masked-minus-original intervals use 10,000 paired patient
resamples. Explicit-language removal therefore weakens ranking and
efficiency without collapsing the frozen controller. However,
deterministic masking cannot remove every implicit cue or establish
that no diagnostic leakage remains. The versioned artifacts include
the regex list, replacement rules, split-level match counts, and
immutable input/output hashes.

\section{Expanded Paired Bootstrap and Trajectory Details}
\label{app:paired}

All paired entries report \crosmix{}, analytic expectation minus the named
comparator. Intervals are 2.5--97.5 percentile intervals from 10,000
patient-level paired bootstrap resamples. All states and outcomes from
one patient are resampled together; states are never resampled
independently. Replicates with no autonomous diagnoses are omitted only
when the selective-risk difference is undefined.

\begin{table}[H]
\caption{\crosmix{}, analytic expectation minus comparator, using 10,000
patient-level resamples: selective risk, coverage, and error mass.}
\label{tab:paired_diagnostic}
\centering
\scriptsize
\resizebox{\linewidth}{!}{%
\begin{tabular}{lrrr}
\toprule
Comparator & Selective risk & Coverage & Error mass \\
\midrule
Confidence threshold
& $-.0877\ [-.1327,-.0439]$
& $-.0077\ [-.0568,.0414]$
& $-.0711\ [-.1115,-.0313]$ \\
LA native + defer
& $-.0633\ [-.1051,-.0220]$
& $+.0141\ [-.0324,.0609]$
& $-.0466\ [-.0833,-.0097]$ \\
\fixedsequenceltt{}
& $+.0043\ [-.0193,.0280]$
& $+.0277\ [-.0032,.0588]$
& $+.0079\ [-.0122,.0281]$ \\
\crosdet{}
& $-.0059\ [-.0132,.0010]$
& $-.0213\ [-.0315,-.0120]$
& $-.0084\ [-.0155,-.0019]$ \\
\uniformmixture{}
& $-.0031\ [-.0104,.0042]$
& $+.0023\ [-.0066,.0112]$
& $-.0021\ [-.0085,.0044]$ \\
Empirical ERM
& $-.0328\ [-.0632,-.0022]$
& $-.0622\ [-.1028,-.0207]$
& $-.0384\ [-.0673,-.0094]$ \\
\bottomrule
\end{tabular}}
\end{table}

\begin{table}[H]
\caption{\crosmix{}, analytic expectation minus comparator: total relative cost,
requested actions, and proxy-miss mass.}
\label{tab:paired_resource}
\centering
\scriptsize
\resizebox{\linewidth}{!}{%
\begin{tabular}{lrrr}
\toprule
Comparator & Total cost & Tests & Proxy-miss mass \\
\midrule
Confidence threshold
& $-12.543\ [-14.953,-10.143]$
& $-2.956\ [-3.424,-2.483]$
& $-.0517\ [-.0878,-.0160]$ \\
LA native + defer
& $-4.837\ [-6.020,-3.654]$
& $-.849\ [-1.004,-.699]$
& $-.0354\ [-.0703,-.0013]$ \\
\fixedsequenceltt{}
& $-6.861\ [-8.065,-5.722]$
& $-1.408\ [-1.604,-1.217]$
& $+.0082\ [-.0109,.0272]$ \\
\crosdet{}
& $-.903\ [-1.135,-.682]$
& $-.195\ [-.228,-.164]$
& $-.0081\ [-.0143,-.0030]$ \\
\uniformmixture{}
& $+.977\ [.766,1.200]$
& $+.175\ [.148,.204]$
& $+.0018\ [-.0037,.0075]$ \\
Empirical ERM
& $+4.067\ [3.247,4.933]$
& $+.679\ [.583,.777]$
& $-.0190\ [-.0444,.0061]$ \\
\bottomrule
\end{tabular}}
\end{table}

These paired comparisons support a narrow cost claim. \crosmix{},
analytic expectation is less costly than \crosdet{}, the deterministic
controller returned by \fixedsequenceltt{}, confidence
thresholding, and native stopping with deferral, but it is more costly
than the \uniformmixture{} and ERM. Its selective-risk difference from
\crosdet{}, the deterministic controller returned by
\fixedsequenceltt{}, and the \uniformmixture{} is unresolved.
ERM is less costly but has higher selective risk and does not pass the
exploratory joint calibration criterion.

\begin{figure}[t]
\centering
\includegraphics[
    width=.84\linewidth
]{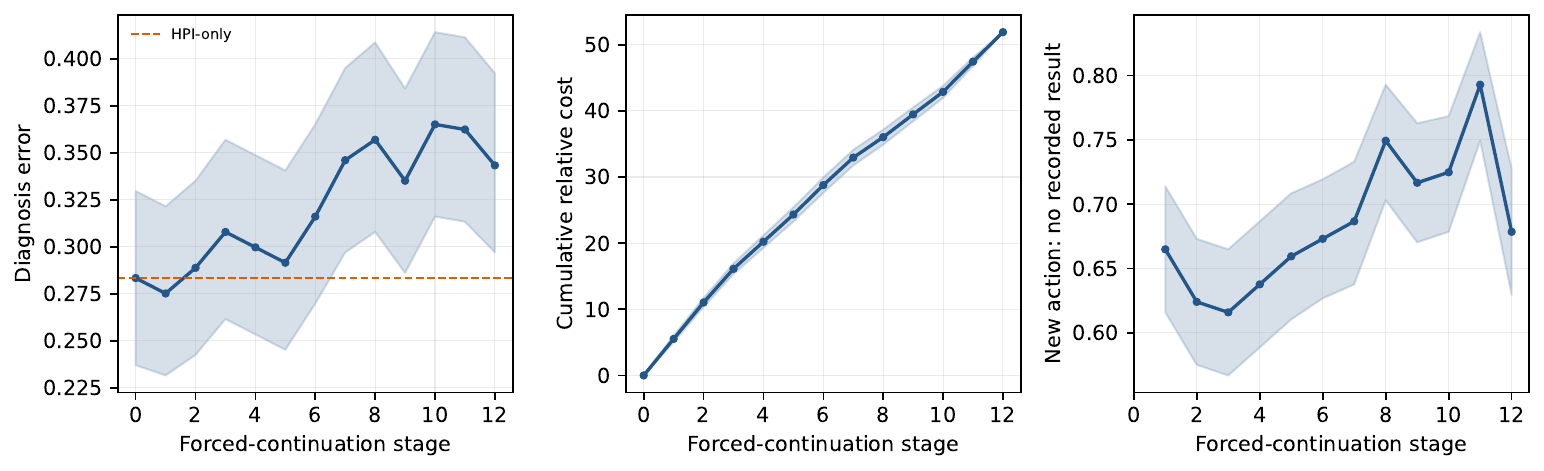}
\caption{Forced-continuation evaluation audit on the same 367 episodes
at every stage. Bands are patient-bootstrap 95\% intervals;
diagnostic error becomes non-monotone as cost and recorded-result
missingness accumulate.}
\label{fig:stagewise}
\end{figure}

\paragraph{Complete forced-continuation trajectory.}
Table~\ref{tab:stagewise_full} reports the numerical trajectory
underlying Figure~\ref{fig:stagewise}. Every row contains the same 367
patients. ``New-action missing'' is the fraction of actions requested
at that stage that return \textsf{NO\_RECORDED\_RESULT}.

\begin{table}[H]
\caption{Forced-continuation evaluation trajectory. All metrics have
patient-bootstrap intervals in the released CSV.}
\label{tab:stagewise_full}
\centering
\scriptsize
\setlength{\tabcolsep}{4.0pt}
\begin{tabular}{rrrrrrr}
\toprule
Stage & $n$ & Error & Macro recall & Cum. cost & Tests
& New-action missing \\
\midrule
0 & 367 & .283 & .621 & .00 & 0 & -- \\
1 & 367 & .275 & .641 & 5.53 & 1 & .665 \\
2 & 367 & .289 & .643 & 11.03 & 2 & .624 \\
3 & 367 & .308 & .610 & 16.12 & 3 & .616 \\
4 & 367 & .300 & .605 & 20.22 & 4 & .638 \\
5 & 367 & .292 & .616 & 24.33 & 5 & .659 \\
6 & 367 & .316 & .593 & 28.77 & 6 & .673 \\
7 & 367 & .346 & .566 & 32.94 & 7 & .687 \\
8 & 367 & .357 & .556 & 36.02 & 8 & .749 \\
9 & 367 & .335 & .575 & 39.47 & 9 & .717 \\
10 & 367 & .365 & .559 & 42.87 & 10 & .725 \\
11 & 367 & .362 & .551 & 47.44 & 11 & .793 \\
12 & 367 & .343 & .571 & 51.90 & 12 & .678 \\
\bottomrule
\end{tabular}
\end{table}

Diagnostic error reaches its minimum after one requested action and
subsequently becomes non-monotone. Because patient composition is
identical across stages, the pattern is not caused by different
patients remaining at later horizons. The degradation occurs alongside
increasing cumulative cost and high recorded-result missingness. It
therefore characterizes the frozen backbone and retrospective
benchmark rather than establishing that clinical testing is generally
harmful.

\section{Exact Bounds and Frozen Frontiers}
\label{app:exact}

\paragraph{Exact component and joint $p$-values.}
Table~\ref{tab:exact_pvalues} reports the risk and coverage components
of the intersection--union test. The joint value is
$p_j=\max(p_{R,j},p_{C,j})$ and is small only when both requirements
receive sufficient evidence.

\begin{table}[H]
\caption{Exact component and joint calibration $p$-values, computed
from unrounded episode counts.}
\label{tab:exact_pvalues}
\centering
\small
\setlength{\tabcolsep}{4.5pt}
\begin{tabular}{lrrrrl}
\toprule
Calibration configuration / controller & Cal auto/err & $p_R$ & $p_C$ & Joint $p$ & Rule \\
\midrule
\singlecandidate{}
& 282/48 & .000850 & .002101 & .002101 & Single \\
\fixedsequenceltt{}
& 275/43 & .000116 & .021160 & .021160 & Fixed sequence \\
\bonferronitesting{}
& 288/48 & .000442 & .000167 & .000442 & Bonferroni \\
\crosmix{}, realized draw
& 280/51 & .004299 & .004354 & .004354 & Single \\
\bottomrule
\end{tabular}
\end{table}

The \singlecandidate{} and the separately frozen \crosmix{} draw are
each tested once at level .05. \fixedsequenceltt{} follows its
prespecified order and passes the reported candidate when that
candidate is reached. \holmtesting{} returns the same deterministic
controller as \fixedsequenceltt{} in this run. \bonferronitesting{}
uses threshold $.05/12=.00417$; its returned controller has unrounded
raw joint value 0.00044157, giving adjusted
$p=\min\{1,12p\}=0.00530$. These are exploratory calibration
calculations because the calibration labels are not prospectively
untouched.

\paragraph{Mixture weights and randomization stability.}
The frozen mixture places weights 0.322, 0.044, and 0.633, rounded to
three decimal places, on
$(h=1,\tau=0.3805)$,
$(h=1,\tau=0.4506)$, and
$(h=3,\tau=0.3286)$, respectively. The printed weights sum to 0.999
because of rounding; sampling and analytic integration use the
full-precision normalized weights.

The displayed \crosmix{} realized draw produces 280 autonomous diagnoses
with 51 errors on calibration and 288 autonomous diagnoses with 47
errors on evaluation. Exact testing uses these realized episode-level
outcomes. Analytic mixture quantities instead integrate each episode's
component contributions over the frozen weights.

\begin{table}[H]
\caption{Two evaluation modes of \crosmix{} and stability across
1,000 episode-wise mixture seeds.
Brackets in the final row give 2.5--97.5 percentiles.}
\label{tab:mixture_stability}
\centering
\small
\setlength{\tabcolsep}{4.5pt}
\begin{tabular}{lrrrr}
\toprule
Summary & Selective risk & Coverage & Cost & Tests \\
\midrule
\crosmix{}, realized draw
& .163 & .785 & 5.68 & .679 \\
\crosmix{}, analytic expectation
& .169 & .788 & 5.57 & .679 \\
1,000-seed mean [percentiles]
& .169 [.159,.179]
& .788 [.774,.801]
& 5.56 [5.22,5.89]
& .679 [.627,.728] \\
\bottomrule
\end{tabular}
\end{table}

All 1,000 evaluation realizations have point estimates below the
numerical risk target and above the numerical coverage target. This
post-freeze analysis measures sensitivity to the episode-wise random
draws; it is not used to select a favorable seed and does not create
1,000 independent calibration experiments.

\paragraph{One-sided exact bounds.}
The following bounds use the realized episode-level outcomes for
randomized policies. Evaluation bounds are descriptive because the
evaluation labels were previously viewed.

\begin{table}[H]
\caption{One-sided 95\% Clopper--Pearson bounds. The
support-restricted deterministic choice equals \crosdet{} and is not
duplicated.}
\label{tab:exact_bounds}
\centering
\scriptsize
\setlength{\tabcolsep}{3.2pt}
\begin{tabular}{lrrrr}
\toprule
& \multicolumn{2}{c}{Calibration}
& \multicolumn{2}{c}{Evaluation (descriptive)} \\
Policy & Risk upper & Cov. lower & Risk upper & Cov. lower \\
\midrule
\crosmix{}, realized draw
& .2243 & .7235 & .2033 & .7464 \\
\crosdet{}
& .2258 & .7321 & .2154 & .7723 \\
\fixedsequenceltt{}
& .1970 & .7093 & .2058 & .7207 \\
LA-CDM native
& .3556 & .9919 & .3500 & .9919 \\
Native + defer
& .3090 & .7493 & .2774 & .7350 \\
Confidence threshold
& .3150 & .7723 & .3024 & .7579 \\
Empirical ERM
& .3037 & .8160 & .2430 & .8160 \\
\bottomrule
\end{tabular}
\end{table}

The bounds concern marginal population risk and coverage and do not
imply disease-class or demographic control.

\paragraph{Frozen evaluation frontiers.}
Figure~\ref{fig:frontier} displays the complete candidate grids for
the supported risk-score families. The marked operating points and
matched-coverage policies are selected using development or selection
data rather than visually favorable evaluation outcomes.

\begin{figure}[H]
\centering
\includegraphics[
    width=\linewidth
]{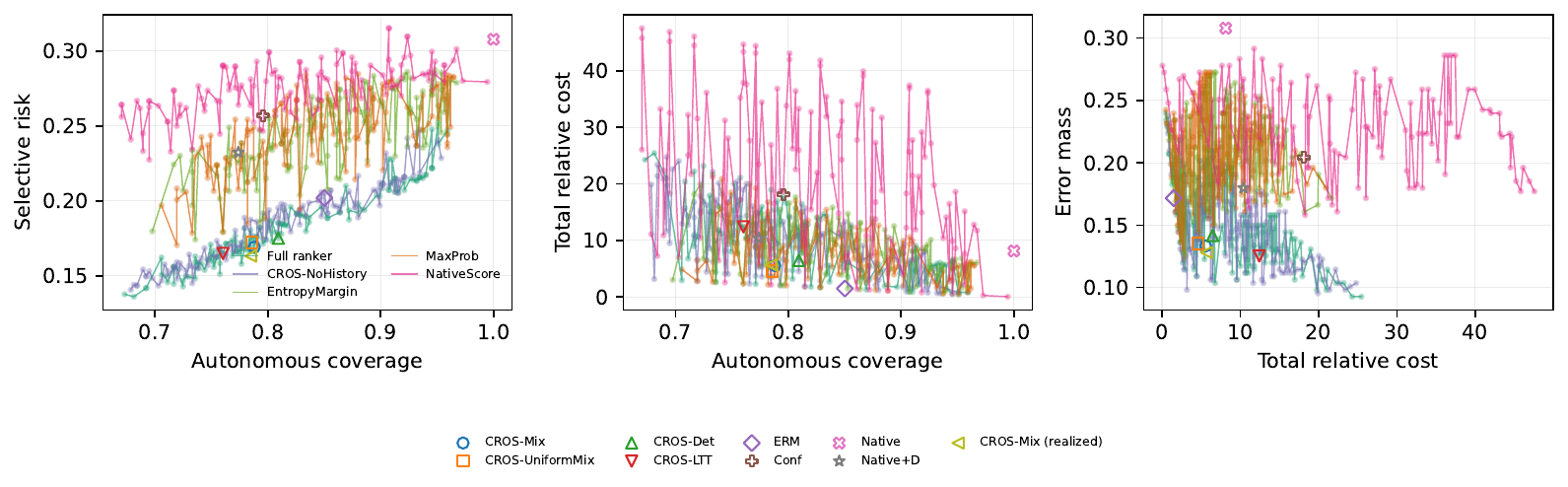}
\caption{Exploratory frozen frontiers. Curves contain the complete
candidate grid for the five supported score families; marked operating
points and the 0.80 matched-coverage policies were selected on
development or selection data.}
\label{fig:frontier}
\end{figure}

\section{Cost, Missingness, and Resplit Details}
\label{app:cost_details}

\paragraph{Frozen-decision cost sensitivity.}
The cost analysis changes the accounting vector while retaining the
original backbone trajectories, stopping decisions, deferral outcomes,
thresholds, and mixture weights. Risk, coverage, and requested actions
therefore remain unchanged.

\begin{table}[H]
\caption{Frozen-decision scenario costs. Low review, primary, and high
review use deferral penalties 5, 10, and 25, respectively. Imaging
$2\times$ and missing $2\times$ double the corresponding action-cost
components.}
\label{tab:cost_sensitivity}
\centering
\scriptsize
\resizebox{\linewidth}{!}{%
\begin{tabular}{lrrrrr}
\toprule
Policy
& Low review
& Primary
& High review
& Imaging $2\times$
& Missing $2\times$ \\
\midrule
\crosmix{}, analytic expectation
& 4.51 & 5.57 & 8.75 & 6.06 & 8.31 \\
\fixedsequenceltt{}
& 11.23 & 12.43 & 16.02 & 13.65 & 20.54 \\
\crosdet{}
& 5.51 & 6.47 & 9.33 & 7.17 & 10.08 \\
Confidence threshold
& 17.09 & 18.11 & 21.18 & 19.25 & 32.09 \\
LA-CDM native
& 8.14 & 8.14 & 8.14 & 9.13 & 14.85 \\
LA native + defer
& 9.27 & 10.40 & 13.80 & 11.39 & 17.11 \\
Empirical ERM
& .75 & 1.50 & 3.75 & 1.50 & 1.50 \\
\bottomrule
\end{tabular}}
\end{table}

For \crosmix{}, analytic expectation, doubling laboratory costs gives
total cost 5.77.
Across the full deferral-penalty grid
$\{5,10,15,20,25,30\}$, the break-even penalty against native stopping
is approximately 22.15. \crosmix{}, analytic expectation remains less
costly than \crosdet{}, the deterministic controller returned by
\fixedsequenceltt{}, confidence thresholding, native stopping, and
native stopping with deferral in the primary, imaging-sensitive, and
missingness-sensitive scenarios. Native stopping becomes less costly
under the deferral penalty of 25, while ERM remains less costly
throughout but fails exploratory calibration.

\begin{figure}[H]
\centering
\includegraphics[
    width=\linewidth
]{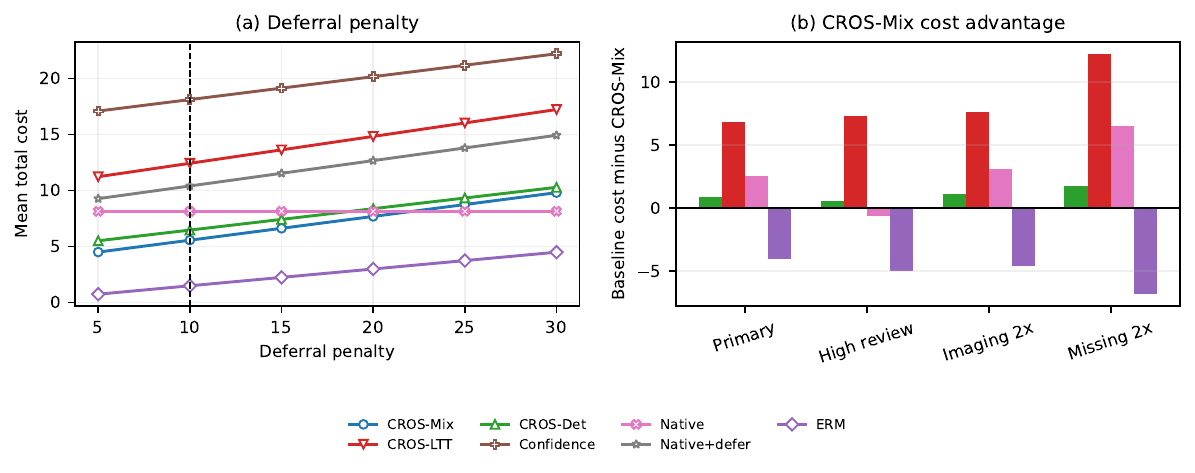}
\caption{Frozen-decision sensitivity. Positive bars in (b) favor
\crosmix{}. Risk, coverage, and requested actions remain fixed; only
cost accounting changes.}
\label{fig:cost_sensitivity}
\end{figure}

\paragraph{Action availability and informative missingness.}
Action availability varies because an action is marked available only
when a corresponding result is found in the retrospective record.

\begin{table}[H]
\caption{Evaluation action availability and final common-path
diagnosis error. Small available or missing-result denominators should
not be overinterpreted.}
\label{tab:action_availability}
\centering
\scriptsize
\begin{tabular}{lrrrr}
\toprule
Action & Available & No result & Err.\ avail. & Err.\ missing \\
\midrule
ECG & 1.9\% & 98.1\% & 28.6\% & 34.4\% \\
MRI & 3.8\% & 96.2\% & 35.7\% & 34.3\% \\
Microbiology & 6.0\% & 94.0\% & 31.8\% & 34.5\% \\
X-ray & 10.1\% & 89.9\% & 21.6\% & 35.8\% \\
Ultrasound & 12.3\% & 87.7\% & 42.2\% & 33.2\% \\
CT & 13.4\% & 86.6\% & 32.7\% & 34.6\% \\
Urinalysis & 18.0\% & 82.0\% & 33.3\% & 34.6\% \\
Metabolic panel & 22.6\% & 77.4\% & 36.1\% & 33.8\% \\
Lipase & 31.9\% & 68.1\% & 41.9\% & 30.8\% \\
Hepatic panel & 62.9\% & 37.1\% & 39.8\% & 25.0\% \\
Recheck vitals & 97.3\% & 2.7\% & 34.2\% & 40.0\% \\
CBC & 97.5\% & 2.5\% & 34.1\% & 44.4\% \\
\bottomrule
\end{tabular}
\end{table}

A logistic model using only action-availability indicators has
diagnosis macro-AUROC 0.597 and state-error AUROC 0.535. Missingness
therefore contains weak information about diagnosis and backbone
error. This is consistent with the record encoding clinicians'
historical ordering behavior, although it does not identify the causal
mechanism producing the missingness.

\paragraph{Development-resplit audit.}
For each of 20 development-only resplits, the 935/165 division is
recreated, the ranker is refitted, and the policies are redesigned
using cached backbone trajectories. The outer calibration and
evaluation cohorts remain unchanged.

\begin{table}[H]
\caption{Twenty development-only repeated splits. Wilson intervals are
reported for rates; metric ranges summarize all available policies.
Infeasible mixture runs remain failures in the reported rates.}
\label{tab:resplit_audit}
\centering
\scriptsize
\resizebox{\linewidth}{!}{%
\begin{tabular}{lrrrrrr}
\toprule
Policy
& Available [95\% CI]
& Sel.\ feasible
& Nominal pass
& Risk mean [range]
& Cov.\ mean [range]
& Cost mean [range] \\
\midrule
\crosdet{}
& 20/20 [.839,1]
& 15/20 [.531,.888]
& 15/20 [.531,.888]
& .185 [.155,.221]
& .829 [.752,.902]
& 5.17 [.98,14.64] \\
\crosmix{}
& 15/20 [.531,.888]
& 15/20 [.531,.888]
& 6/20 [.145,.519]
& .194 [.163,.231]
& .843 [.789,.919]
& 2.33 [.81,4.45] \\
\bottomrule
\end{tabular}}
\end{table}

The \crosmix{} LP is infeasible for seeds 20260911, 20260913,
20260915, 20260923, and 20260924. Failures remain in the denominator
rather than being discarded. The full run-level distribution,
including unsuccessful runs, accompanies the aggregate artifacts.
Because the same outer calibration and evaluation cohorts are reused,
these scores are dependent exploratory checks rather than 20
independent confirmations. The audit was not used to replace or modify
the primary frozen policy.

\section{Disease-Class Audit}
\label{app:subgroups}

Table~\ref{tab:class_audit} reports the complete disease-class results
for the prespecified \crosmix{} realized-draw evaluation seed. The values
are descriptive, are not multiplicity-controlled, and are not
class-conditional certificates.

\begin{table}[H]
\caption{Complete nine-class audit for \crosmix{}, realized draw.}
\label{tab:class_audit}
\centering
\scriptsize
\setlength{\tabcolsep}{3.3pt}
\begin{tabular}{lrrrrrr}
\toprule
Reference class
& $n$
& Autonomous
& Errors
& Coverage
& Selective risk
& Error mass \\
\midrule
Appendicitis
& 89 & 85 & 3 & .955 & .035 & .034 \\
Biliary disease
& 61 & 41 & 3 & .672 & .073 & .049 \\
Bowel obstruction
& 51 & 42 & 2 & .824 & .048 & .039 \\
Diverticulitis
& 26 & 18 & 9 & .692 & .500 & .346 \\
Gastroenteritis/colitis
& 44 & 32 & 6 & .727 & .188 & .136 \\
Nonspecific abdominal pain
& 17 & 10 & 10 & .588 & 1.000 & .588 \\
Pancreatitis
& 36 & 27 & 9 & .750 & .333 & .250 \\
Renal colic
& 10 & 8 & 1 & .800 & .125 & .100 \\
Urinary infection
& 33 & 25 & 4 & .758 & .160 & .121 \\
\bottomrule
\end{tabular}
\end{table}

Coverage is below 70\% for biliary disease, diverticulitis, and
nonspecific abdominal pain. Selective error is 0.500 for
diverticulitis and 1.000 for the small nonspecific-abdominal-pain
subgroup. The denominators are too small to support simultaneous
class-level certificates at $\alpha=.25$ while maintaining useful
coverage.

These results do not contradict the marginal population calculation:
Theorem~\ref{thm:certificate} concerns aggregate selective risk and
coverage unless subgroup constraints are explicitly incorporated into
the frozen testing family. A future confirmatory design should first
certify marginal risk, then test prespecified clinically meaningful or
high-harm strata with an allocated multiplicity budget. Groups lacking
adequate sample size should be reported as unsupported rather than
safe.

\section{Frozen Protocol and Confirmatory Extension}
\label{app:protocol}

\paragraph{Backbone implementation.}
The diagnostic backbone is Qwen2.5-7B-Instruct
\citep{qwen2024qwen25} with LoRA adaptation
\citep{hu2022lora} using rank 8, scaling parameter 16, and dropout
0.1. Following the official LA-CDM implementation at frozen repository
commit \textsf{3f435a1}, adaptation proceeds through decision-agent
GRPO, hypothesis-agent supervised fine-tuning, and
confidence-calibration GRPO. Training runs for four epochs and 3,740
optimizer steps with seed 269.

The rehashed training manifest references only the 935-episode
ranker-training partition and the 165-episode development-selection
partition; the canonical calibration and evaluation files are not
loaded by the backbone-training pipeline. The resulting model produces
the diagnosis distribution, diagnosis proposal, next-action proposal,
native stopping score, and other logged quantities subsequently used
by all stopping controllers. Our implementation is an
official-code-derived transfer of the LA-CDM training structure and
prompt conventions. It should not be interpreted as a
prompt-equivalent reproduction of LA-CDM's original free-form
trajectories.

\paragraph{Common-path stopping protocol.}
For each episode, the frozen backbone is first evaluated along a
maximum forced-continuation trajectory of 12 action opportunities.
At each stage, the trace records the current information state,
diagnosis probabilities, proposed diagnosis, native stopping score,
proposed next action, returned observation, and accumulated relative
cost. Continuing reveals the result associated with the next
backbone-proposed action, including
\textsf{NO\_RECORDED\_RESULT} when no corresponding result is found.
No stopping controller substitutes a different action.

All compared controllers are then applied to the same stored
backbone trajectory. A controller determines only whether to accept
the current diagnosis, continue with the backbone proposal, or defer.
Its requested-action count and cost are truncated at its terminal
stage, with the deferral penalty added when applicable. Thus paired
differences among the primary controllers and comparators isolate stopping and deferral
decisions while holding the backbone's diagnoses, action proposals,
and potential recorded observations fixed. They do not compare
alternative test-acquisition strategies.

\paragraph{Stopping comparators.}
The common-path comparison separates the two formal \cros{}
controllers from baselines, testing procedures, and ablations as
follows.

\paragraph{Fixed-information comparators.} HPI-only accepts the stage-zero diagnosis without requesting an action. Full workup continues through all 12 action opportunities before accepting the terminal diagnosis. The fixed-stage comparator accepts at one common stage chosen on the development selection split. These study-defined controls compare fixed information budgets and are not implementations of external methods.

\paragraph{Confidence-based comparators.} The confidence-threshold controller stops when the backbone's maximum diagnosis probability crosses its selection-frozen threshold and otherwise defers at its horizon. The maximum-probability ranker directly uses maximum class probability as the state-ranking score, while the entropy-margin ranker uses predictive entropy and the gap between the two largest diagnosis probabilities. These are study-implemented confidence heuristics inspired by standard selective prediction and abstention \citep{geifman2017selective,geifman2019selectivenet}; they are not exact reproductions of those papers' training procedures.

\paragraph{Native-agent comparators.} LA-CDM native follows the backbone's own stopping output. Native plus defer applies an additional selection-frozen confidence requirement to the diagnosis produced at the native stopping stage; cases failing that requirement are deferred. The native component is an official-code-derived transfer of LA-CDM \citep{baniharouni2026lacdm} to the fixed common-path setting, not a prompt-equivalent reproduction of its free-form trajectories. Native plus defer is our derivative baseline combining that native stop with confidence-based deferral; it is not a named method from the LA-CDM paper.

\paragraph{Empirical cost minimization.} Empirical ERM selects the policy with lowest empirical mean cost on the policy-selection split without requiring the joint selective-risk and autonomous-coverage criterion used by \cros{}. It tests whether empirical cost optimization alone sacrifices risk control.

\paragraph{Formal \cros{} controllers.} \crosdet{} is the least-cost deterministic threshold-and-horizon candidate satisfying the selection-stage design constraints. \crosmix{} is the LP-optimized episode-wise mixture defined by Eq.~(\plaineqref{eq:lp}) over the frozen deterministic family. In the primary freeze, the best component in the optimized support is identical to \crosdet{} and is not displayed separately.

\paragraph{Calibration and multiplicity procedures.} The \singlecandidate{} evaluates one policy fixed before calibration. \fixedsequenceltt{} applies the prespecified LTT order \citep{angelopoulos2025ltt}; \holmtesting{} applies Holm's step-down rule \citep{holm1979simple}; and \bonferronitesting{} applies the classical Bonferroni threshold to the same frozen 12-policy family. Holm testing returns the same deterministic controller as fixed-sequence LTT in this run, whereas Bonferroni testing returns a different deterministic controller. These procedures alter which candidate is tested or returned, not its ranker, backbone, or basic stopping rule. Pareto Testing is related multi-risk methodology \citep{laufergoldshtein2023pareto}, but is not a separate controller in the reported table.

\paragraph{Ablations.} The \nohistoryranker{} removes stage, missing-result fraction, cumulative cost, and latency while retaining probability-derived features; it tests whether history features improve state-error ranking and stopping. The \nativescoreranker{} directly uses the backbone's native stopping score; it tests whether the learned ranker improves on that score. The \uniformmixture{} assigns equal probability to the same frozen support as \crosmix{}; it tests whether optimized weights improve on support selection alone. These ablations neither retrain nor alter the diagnostic backbone.

\paragraph{Mixture evaluation modes and legacy identifiers.} \crosmix{}, analytic expectation integrates each episode's support contributions, whereas \crosmix{}, realized draw uses one frozen episode-wise component assignment. They are evaluations of one controller, not separate methods. Released machine-readable files retain \texttt{BESTDET}, \texttt{SUPPORTBEST}, \texttt{OPTIMIZEDMIX}, and \texttt{UNIFORMMIX} for artifact compatibility; their display meanings are \crosdet{}, the support component identical to \crosdet{} in this run, \crosmix{}, and the \uniformmixture{}, respectively.

The exact thresholds, horizons, 12-policy testing family, candidate order, and mixture weights are frozen using only the designated development-selection data and are stored in the run manifest. Comparators that select a threshold, stage, horizon, or policy use the same development-selection partition.

\paragraph{Episode-wise mixture randomization.}
For a frozen mixture with weights $w$, one component policy is sampled independently at the beginning of each episode and is retained for the entire episode. Sampling is not conditioned on patient characteristics, backbone predictions, intermediate observations, or calibration outcomes. The base policy seed is 20260902; the displayed calibration and evaluation realizations use the frozen split-specific seeds 20260904 and 20260905, respectively.

The exact binomial test treats \crosmix{}, realized draw as one frozen
policy and uses its realized autonomous and error indicators.
\crosmix{}, analytic expectation instead averages each episode's
contributions over the frozen component weights. It reduces Monte Carlo
noise in descriptive cost and paired comparisons, but is not
substituted for the realized Bernoulli outcomes in the exact calibration
test. The separate 1,000-seed analysis measures randomization stability
and is not used to select a favorable seed. Drawing one global component
and applying it to every calibration episode would create shared
dependence and would not justify the same binomial calculation.

\paragraph{Metrics and uncertainty.}
For a controller producing $M$ autonomous diagnoses and $E$ errors among $n$ episodes, the primary diagnostic quantities are selective error $E/M$, autonomous coverage $M/n$, and error mass $E/n$. Selective error is undefined when $M=0$. Error mass is reported alongside selective error so that changes in the number of deferred cases do not conceal the absolute frequency of autonomous mistakes.

Resource outcomes are total relative cost and the number of attempted actions. An attempted action is counted even when it returns \textsf{NO\_RECORDED\_RESULT}. Proxy-miss mass is the proportion of episodes with a wrong autonomous prediction whose reference label is one of the prespecified disease-specific classes other than gastroenteritis/colitis or nonspecific abdominal pain. This endpoint is descriptive and is not a clinician-adjudicated measure of harm or diagnostic urgency.

Paired contrasts use 10,000 patient-level bootstrap resamples. All states and outcomes belonging to the same patient are resampled together; state-level resampling is never used. Replicates with no autonomous diagnoses are excluded only from contrasts for which selective error is mathematically undefined. Analytic mixture comparisons integrate each patient's component-level contributions before patient-level resampling. A separate 1,000-seed analysis describes the variability of finite mixture realizations.

For deterministic policies and the prespecified \crosmix{} realized
draw, we also report one-sided 95\% Clopper--Pearson upper bounds for
selective risk and lower bounds for coverage. Evaluation bounds remain
descriptive in the present study because those labels were previously
viewed. Frozen-decision sensitivity analyses change the cost vector
without rerunning the backbone, refitting the ranker, reconstructing
thresholds, or changing controller decisions.

\paragraph{Freeze order.}
The recorded pipeline proceeds in the following order: backbone adaptation on development data; episode-wise cross-fitting and final risk-ranker fitting on the 935-patient ranker-training subset; threshold, candidate-family, candidate-order, and mixture design on the disjoint 165-patient policy-selection subset; and finally calibration and evaluation. Candidate ranking and the 12-policy multiple-testing order use selection data only within the recorded pipeline.

The policy, ranker, thresholds, candidate order, mixture weights, cost vector, evaluation code, and randomization mechanism are fixed before the reported calibration and evaluation computations under the current freeze. This procedural freeze makes the reported analyses reproducible, but it does not restore prospective independence after labels have previously been accessed.

\paragraph{Recorded provenance.}
The run stores model and adapter identifiers, repository commit, training and generation configuration, dataset and split hashes, label code, action-cost vector, ranker configuration, all candidate policies, the 12-policy testing family and its order, mixture weights, base seed 20260902, split-specific realization seeds, evaluator version, and per-episode traces. Thresholds are stored numerically rather than recomputed during calibration or evaluation.

Two independently implemented analysis scripts reproduce the reported summary files and verify split disjointness. The audit recomputes zero patient intersections among the 935 ranker-training, 165 policy-selection, 367 calibration, and 367 evaluation partitions. Because cohort construction retains the earliest qualifying stay before partitioning, each episode corresponds to one patient. Derived analysis artifacts contain no patient identifiers.

\paragraph{Separate post-freeze exploratory branch.}
A separately constructed cohort branch uses partitions of
543/85/190/199 episodes and is not pooled with the canonical cohort, used to modify its policies, or included in its primary tables. In that branch, Primary-RobustDet is point-estimate dominated by the branch-specific \crosmix{} in selective risk, coverage, and cost. Full-Ensemble-UCB has AURC 0.1290, compared with 0.1289 for the \nohistoryranker{}, and therefore does not improve the ranking result. Among the corresponding contrasts, only the cost contrast is statistically resolved. We did not continue tuning either variant after observing these results. This branch is reported as negative exploratory evidence, not as an independent confirmation of the canonical analysis.

\paragraph{Required confirmatory freeze.}
Before accessing a new confirmatory cohort, a signed and timestamped manifest must fix the cohort dates, inclusion and exclusion rules, one-episode-per-patient construction, diagnosis-label mapping, timestamp interpretation, handling of missing or unavailable tests, adjudication procedures, harm and fairness strata, action and deferral costs, model and adapter hashes, generation settings, ranker, thresholds, candidate family and order, mixture weights and episode-wise randomization mechanism, $(\alpha,\gamma,\delta)$, multiplicity graph, sample size, stopping rule for data collection, and executable analysis code.

Calibration labels must remain inaccessible to model fitting, candidate construction, policy ordering, mixture optimization, seed selection, and sample-size revision. Any amendment made after calibration outcomes are accessed creates a new exploratory analysis and requires another untouched cohort for a confirmatory claim. Freezing code after examining labels does not recreate the independence assumed by Theorem~\ref{thm:certificate}.

\paragraph{Confirmatory decision rule.}
The confirmatory calibration set should be opened once and evaluated using the preregistered candidate family and testing graph. A policy is certified only if its union null is rejected by the frozen family-wise-error procedure at level $\delta$. If no candidate is rejected, the confirmatory outcome is that no policy is certified; the investigators should not replace it with the empirically best failed candidate.

Any subsequent evaluation cohort should be used only after the calibration decision has been finalized. Its role is to estimate performance, subgroup behavior, resource use, and distribution shift, rather than to revise the certified policy or repeat the certificate test.

\paragraph{Sample-size planning.}
Power must be determined before the confirmatory freeze using development estimates and prespecified worst-case margins. Planning should simulate the joint distribution of autonomous diagnoses and autonomous errors because the selective-risk denominator is itself random. It should also reproduce the intended candidate order, intersection-union tests, multiplicity rule, mixture randomization, and any subgroup or harm constraints.

Although $n=367$ is sufficient to reject several candidate nulls in the present exploratory calculations, it does not establish adequate power under smaller risk or coverage margins, temporal shift, multiplicity, subgroup constraints, or clinician-adjudicated harm endpoints. Sample size must follow the complete testing graph rather than repeated inspection of certificate outcomes. Early stopping, sample-size extension, or repeated testing requires a separately valid sequential design.

\paragraph{End-to-end agent study.}
A secondary experiment should allow each agent to choose its own actions rather than restricting all methods to a common trajectory. The comparison should include fixed-stage, confidence, ERM, myopic value of information, official LA-CDM, LTT, \crosdet{}, \crosmix{}, and compatible released implementations of AgentClinic, MedChain, and DxChain. It should report invalid or unavailable actions, repeated actions, token use, latency, diagnostic accuracy, selective error, error mass, autonomous coverage, resource cost, and trace validity. Clinician adjudication should additionally assess whether requested tests and terminal decisions are medically appropriate.

Because independently acting agents generate different histories, their outcomes cannot be interpreted as paired stopping comparisons on a common information path. Such an experiment evaluates the combined acquisition, reasoning, stopping, and deferral system and therefore addresses a broader question than the primary experiment.

\paragraph{Scope of the theoretical claim.}
Theorem~\ref{thm:certificate} controls the probability of certifying an invalid frozen policy only under the stated exchangeability and independent-calibration assumptions. It does not establish causal clinical benefit, correctness of the reference labels, realism of the relative cost vector, optimality outside the frozen candidate family, subgroup or hospital-level robustness, or deployment safety.

Repeated episodes from the same patient, clinician, or institution would require the sampling unit and exchangeability assumptions to be redefined, together with a corresponding cluster-aware calibration procedure. Settings in which an action changes patient physiology, future observability, clinical management, or documentation require a causal or interactive environment rather than the present logged trajectory. These limitations define the boundary of the theoretical
and empirical claims.

\end{document}